\documentclass[a4paper,11pt]{article}

\usepackage{amsmath, amsthm, amssymb, amscd, mathtools}
\usepackage{algorithm2e}
\usepackage{graphicx}
\usepackage{bm}
\usepackage{tikz}
\usepackage{dsfont}
\usepackage{bbm}
\usepackage{tikzsymbols}
\usepackage[breaklinks=true]{hyperref}

\newcommand{\bv}{\mathbf{v}}

\newcommand{\N}{\mathbb{N}}
\usepackage[nottoc]{tocbibind}
\usepackage{titlesec}
\usepackage{longtable}
\usepackage{multirow}

\newcommand{\eps}{\varepsilon}
\newcommand{\beps}{\bm{\varepsilon}}

\newcommand{\argmin}{\mathop{\rm arg\min}}

\newcommand{\R}{\mathbb{R}}

\newcommand{\ca}{\mathcal }
\newcommand{\tn}{\, \textnormal}

\newcommand{\beq}{\begin{equation}\label }
\newcommand{\eeq}{\end{equation} }
\newcommand{\bal}{\begin{align*} }
\newcommand{\ve}{\varepsilon}

\newcommand{\bs}{\boldsymbol}
\newcommand{\eal}{\end{align*} }

\newtheorem{thm}{Theorem}[section]
\newtheorem{lem}[thm]{Lemma}

\newtheorem{defi}[thm]{Definition}

\newcommand{\E}{\mathbb{E}}

\newcommand{\bbeta}{{\boldsymbol{\beta}}}

\newcommand{\blambda}{{\boldsymbol{\lambda}}}

\newcommand{\bX}{\mathbf{X}}

\newcommand{\Fm}{\ca F_{\textsc{\tiny MARS}}}

\newcommand{\Fs}{\ca F_{\textsc{\tiny  FS}}}
\newcommand{\Mult}{\operatorname{Mult}}

\newcommand{\footnoteremember}[2]{
  \footnote{#2}
  \newcounter{#1}
  \setcounter{#1}{\value{footnote}}
}
\newcommand{\footnoterecall}[1]{
  \footnotemark[\value{#1}]
}

\title{abc}
\begin{document}

\title{A comparison of deep networks with ReLU activation function and linear spline-type methods}

\author{Konstantin Eckle\footnoteremember{b}{Leiden University, Mathematical Institute, Niels Bohrweg 1, 2333 CA Leiden,
		The Netherlands, k.j.eckle@math.leidenuniv.nl; schmidthieberaj@math.leidenuniv.nl 
 }\and   Johannes Schmidt-Hieber\footnoterecall{b}}

\date{}
\maketitle

\begin{abstract}
Deep neural networks (DNNs) generate much richer function spaces than shallow networks. Since the function spaces induced by shallow networks have several approximation theoretic drawbacks, this explains, however, not necessarily the success of deep networks. In this article we take another route by comparing the expressive power of DNNs with ReLU activation function to linear spline methods. We show that MARS (multivariate adaptive regression splines) is improper learnable by DNNs in the sense that for any given function that can be expressed as a function in MARS with $M$ parameters there exists a multilayer neural network with $O(M \log (M/\eps))$ parameters that approximates this function up to sup-norm error $\eps.$ We show a similar result for expansions with respect to the Faber-Schauder system. Based on this, we derive risk comparison inequalities that bound the statistical risk of fitting a neural network by the statistical risk of spline-based methods. This shows that deep networks perform better or only slightly worse than the considered spline methods. We provide a constructive proof for the function approximations.
\end{abstract}

%
%

\newpage

\section{Introduction}

Training of DNNs is now state of the art for many different complex tasks including detection of objects on images, speech recognition and game playing. For these modern applications, the ReLU activation function is standard. One of the distinctive features of a multilayer neural network with ReLU activation function (or ReLU network) is that the output is always a piecewise linear function of the input. But there are also other well-known nonparametric estimation techniques that are based on function classes built from piecewise linear functions. This raises the question in which sense these methods are comparable. 

We consider methods that consist of a function space $\mathcal{F}$ and an algorithm/learner/ estimator that selects a function in $\mathcal{F}$ given a dataset. In the case of ReLU networks, $\mathcal{F}$ is given by all network functions with fixed network architecture and the chosen algorithm is typically a version of stochastic gradient descent. A method can perform badly if the function class $\mathcal{F}$ is not rich enough or if the algorithm selects a bad element in the class. The class of candidate functions $\mathcal{F}$ is often referred to as expressiveness of a method.

Deep networks have a much higher expressiveness than shallow networks, cf. \cite{Mhaskar1993, asssd, 123423, Prodmult, Liang2016, Poggio2017}. Shallow ReLU networks have, however, several known drawbacks. It requires many nodes/units to localize and to approximate for instance a function supported on a smaller hypercube, cf. \cite{MR1399382}. A large number of parameters is moreover necessary to approximately multiply inputs. While the literature has mainly focused on the comparison deep versus shallow networks, it is necessary to compare deep networks also with methods that have a similar structure but do not suffer from the shortcomings of shallow networks. 

There is a vast amount of literature on spline based approximation, cf. \cite{chui1988multivariate}, \cite{chui2}, \cite{de1993box}, \cite{micchelli1995mathematical}. For a direct comparison with ReLU networks it is natural to focus on spaces spanned by a piecewise linear function system. In this article we compare deep ReLU networks to well-known spline methods based on classes of piecewise linear functions. More specifically, we consider multivariate adaptive regression splines (MARS) \cite{MR1091842} and series expansion with respect to the Faber-Schauder system \cite{Faber1910, vanderMeulen2017}.

MARS is a widely applied method in statistical learning \cite{Hastie2009}. Formally, MARS refers to the algorithm describing the selection of a function from a specific class of piecewise linear candidate functions called the MARS function class.
Similarly as multilayer neural networks, the MARS function class is build from a system of simple functions and has width and depth like parameters. In contrast to DNNs, MARS also uses tensorization to generate multivariate functions. The Faber-Schauder functions are the integrals of the Haar wavelet functions and are hence piecewise linear. Series expansion with respect to the Faber-Schauder system results therefore in a piecewise linear approximation. Both MARS and Faber-Schauder class satisfy an universal approximation theorem, see Section \ref{sec.FS} for a discussion. 

We show that for any given function that can be expressed as a function in the MARS function class or the Faber-Schauder class with $M$ parameters there exists a multilayer neural network with $O(M \log (M/\eps))$ parameters that approximates this function up to sup-norm error $\eps.$ We also show that the opposite is false. There are functions that can be approximated up to error $\eps$ with $O(\log(1/\eps))$  network parameters  but require at least of order $\eps^{-1/2}$ many parameters if approximated by MARS or the Faber-Schauder class.

The result provides a simple route to establish approximation properties of deep networks. For the Faber-Schauder class, approximation rates can often be deduced in a straightforward way and these rates then carry over immediately to approximations by deep networks.  As another application, we prove that for nonparametric regression, fitting a DNN will never have a much larger risk than MARS or series expansion with respect to the Faber-Schauder class. 
Although we restrict ourselves to two specific spline methods, many of the results will carry over to other systems of multivariate splines and expansions with respect to tensor bases.

The proof of the approximation of MARS and Faber-Schauder functions by multilayer neural network is constructive. This allows to run a two-step procedure which in a first step estimates/learns a MARS or Faber-Schauder function. This is then used to initialize the DNN. For more details see Section \ref{sec31}. 

Related to our work, there are few results in the literature showing that the class of functions generated by multilayer neural networks can be embedded into a larger space. \cite{Zhang2016} proves that neural network functions are also learnable by a carefully designed kernel method. Unfortunately, the analysis relies on a power series expansion of the activation function and does not hold for the ReLU. In \cite{zhang2017}, a convex relaxation of convolutional neural networks is proposed and a risk bound is derived.
 
This paper is organized as follows. Section \ref{Sec2} provides short introductions to the construction of MARS functions and the Faber-Schauder class. Notation and basic results for DNNs are summarized in Section \ref{Sec3}. Section \ref{sec.comparison} contains the main results on approximation of MARS and Faber-Schauder functions by DNNs. In Section \ref{risk}, we derive risk comparison inequalities for nonparametric regression. Numerical simulations can be found in Section \ref{FSR}. Additional proofs are deferred to the end of the article.

{\it Notation:} Throughout the paper, we consider functions on the hypercube $[0,1]^d$ and $\|\cdot\|_\infty$ refers to the supremum norm on $[0,1]^d.$ Vectors are displayed by bold letters, e.g. $\bs{X}, \bs x.$

\section{Spline type methods for function estimation}\label{Sec2}

In this section, we describe MARS and series expansions with respect to the Faber-Schauder system. For both methods we discuss the underlying function space and algorithms that do the function fitting. The construction of the function class is in both cases very similar. We first identify a smaller set of functions to which we refer as function system. The function class is then defined as the span of this function system.

\subsection{MARS - Multivariate Adaptive Regression Splines}
Consider functions of the form
\begin{align}
	h_{I, \bs t}(x_1, \ldots, x_d) = \prod_{j \in I} \big(s_j(x_j-t_j)\big)_+, \quad \bs x=(x_1,\hdots,x_d)^\top\in[0,1]^d,
	\label{eq.MARS_bf}
\end{align}
with $ I \subseteq \{1, \ldots, d\}$ an index set,  $s_j \in \{-1,1\}$ and $\bs t = (t_j)_{j \in I}\in [0,1]^{|I|}$ a shift vector. These are piecewise linear functions in each argument whose supports are hyperrectangles. Suppose we want to find a linear combination of functions $h_{I, \bs t}$ that best explains the data. To this end, we would ideally minimize a given loss over all possible linear combinations of at most $M$ of these functions, where $M$ is a predetermined cut-off parameter. Differently speaking, we search for an element in the following function class.

\begin{defi}[MARS function class]\label{2}
For any positive integer $M$ and any constants $F, C>0,$ define the family of functions
 \begin{align*}
	\Fm(M, C)=\Big\{f=\beta_0+\sum_{m=1}^{M}\beta_m h_m \, : \, h_m \text{ is of the form} \ \eqref{eq.MARS_bf}, |\beta_m|\leq C, \|f\|_\infty \leq F\Big\}.
\end{align*}
All functions are defined on on $[0,1]^d.$ For notational convenience, the dependence on the constant $F$ is omitted. 
\end{defi}


For regression, the least squares loss is the natural choice. Least-squares minimization over the whole function space is, however, computationally intractable. Instead, MARS - proposed by \cite{MR1091842} - is a greedy algorithm that searches for a set of basis functions $h_{I, \bs t}$ that explain the data well. Starting with the constant function, in each step of the procedure two basis functions are added. The new basis functions are products of an existing basis function with a new factor $(x_j-t_j)_+,$ where $t_j$ coincides with the $j$-th component of one of the design vectors. Among all possible choices, the basis functions that lead to the best least-squares fit are selected. The algorithm is stopped after a prescribed number of iterations. It can be appended by a backwards deletion procedure that iteratively deletes a fixed number of selected basis functions with small predictive power. Algorithm \ref{alg.MARS} provides additional details.

\begin{algorithm}
 \DontPrintSemicolon
  \SetAlgoLined
  \SetKwInOut{Input}{Input}\SetKwInOut{Output}{Output}
  \Input{Integer $M'$ and sample $(\bs X_i, Y_i),$ $i=1, \ldots,n$}
  \Output{$2M'+1$ basis functions of the form \eqref{eq.MARS_bf}}
   $\ca S \leftarrow  \{1\}$ \;
  \For{$k$ from $1$ to $M'$}{
   \For{$g \in \ca S$}{
   		$J(g) \leftarrow \{j\, : \, x_j \mapsto g(x_1, \ldots, x_d) \ \text{is constant}\}$\; 
   }
   $(g^*, j^*, t_{j^*}^*) \leftarrow$  argmin over $g \in \ca S,$ $j \in J(g),$  $t_j \in \{X_{ij}, i=1, \ldots, n\}$ of
   the least squares fits for the models
   $\ca S \cup \{ \bs x \mapsto (x_j-t_j)_+g(\bs x), \bs x \mapsto (t_j-x_j)_+g(\bs x)\}.$ \;
   $\ca S \leftarrow \ca S \cup \{ \bs x \mapsto (x_{j^*}-t_{j^*}^*)_+g^*(\bs x), \bs x \mapsto (t_{j^*}^*-x_{j^*})_+g^*(\bs x)\}.$
    }
    \KwRet{$\ca S$}
  \caption{MARS forward selection}\label{alg.MARS}
\end{algorithm}

In the original article \cite{MR1091842} the values for $t_j$ are restricted to a smaller set. The least squares functional can be replaced by any other loss. The runtime of the algorithm is of the order $(M')^3nd$, where the key step is the computation of the argmin which is easily parallelizable.  \cite{friedman1993} considers a variation of the MARS procedure where the runtime dependence on $M'$ is only quadratic. 

There are many higher-order generalization of the MARS procedure. E.g., \cite{MR1091842} proposed also to use truncated power basis functions
\begin{align*}
\big(s_j(x_j-t_j)\big)_+^q,\quad q\in\N,
\end{align*}
to represent $q$th order splines. Another common generalization of the MARS algorithm allows for piecewise polynomial basis functions in each argument of the form
\begin{align}
	h_{K, \bs t}(x_1, \ldots, x_d) = \prod_{j=1}^K \big(s_j(x_{i(j)}-t_j)\big)_+, \quad K\in\N,\;s_j\in\{-1,1\},\;\bs t\in[0,1]^K,
	\label{MARS2}
\end{align}
where $i(\,\cdot\,)$ is a mapping $\{1,\hdots,K\}\rightarrow\{1,\hdots,d\}$. The theoretical results of this paper can be generalized in a straightforward way to these types of higher-order basis functions. If the MARS algorithm is applied to the function {class} generated by \eqref{MARS2}, we refer to this as higher order MARS (HO-MARS).

\begin{figure}[ht]
\begin{center}
	\includegraphics[scale=0.6]{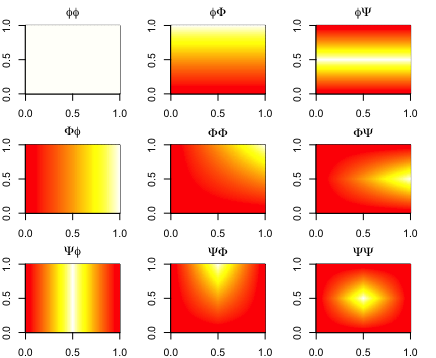} 
\caption{\label{fig.FSfcts} First functions of the bivariate Faber-Schauder {system} displayed as heatmaps.}
\end{center}
\end{figure}

\subsection{Faber-Schauder series expansion}
\label{sec.FS}
The Faber-Schauder {system} defines a family of functions $[0,1]^d\rightarrow\R$ that can be viewed as the {system} of indefinite integrals of the Haar wavelet functions. Faber-Schauder functions are piecewise linear in each component and expansion with respect to the Faber-Schauder {system} results therefore in (componentwise) piecewise linear approximations. 

To describe the expansion of a function with respect to the Faber-Schauder {system}, consider first the univariate case $d=1.$ Let $\phi(x)=1$ and $\psi = \mathbbm{1}(\cdot \in [0,1/2)) - \mathbbm{1}(\cdot \in [1/2,1]).$ The Haar wavelet consists of the functions $\{\phi, \psi_{j,k}, j=0,1,\ldots ; k=0, \ldots, 2^j-1\}$ with $\psi_{j,k} = 2^{j/2}\psi(2^j \cdot -k).$ A detailed introduction to Haar wavelets is given for instance in \cite{wojtaszczyk1997mathematical}.

Define $\Phi(x)= \int _0^x \phi(u) \, du = x$ and $\Psi_{j,k}(x) = \int_0^x \psi_{j,k}(u) \, du.$ The functions $\Psi_{j,k}$ are triangular functions on the interval $[k2^{-j},(k+1)2^{-j}]$. The Faber-Schauder {system} on $[0,1]$ is $\{\phi, \Phi, \Psi_{j,k}, j=0,1,\ldots ; k=0, \ldots, 2^j-1\}.$

\begin{lem}
\label{lem.FS_d=1}
Let $f$ be a function in the Sobolev space $\{f: f' \in L^2[0,1]\}.$ Then, 
\begin{align*}
	f(x) = f(0) + (f(1)-f(0))x + \sum_{j=0}^\infty \sum_{k=0}^{2^j-1} d_{j,k} \Psi_{j,k}(x)
\end{align*}
with $d_{j,k}= - 2^{j/2}(f(\tfrac{k+1}{2^j})- 2f(\tfrac{k+1/2}{2^j})+ f(\tfrac{k}{2^j}))$ and convergence of the sum in $L^2[0,1].$
\end{lem}

\begin{proof}
It is well-known that the Haar wavelet generates an orthonormal basis of $L^2[0,1].$ Since by assumption $f' \in L^2[0,1],$
\begin{align*}
	f'(x) = f(1)-f(0) + \sum_{j=0}^\infty \sum_{k=0}^{2^j-1} d_{j,k} \psi_{j,k}(x)
\end{align*}
with $d_{j,k}=\int f'(u) \psi_{j,k}(u) \, du = -2^{j/2}(f(\tfrac{k+1}{2^j})- 2f(\tfrac{k+1/2}{2^j})+ f(\tfrac{k}{2^j}))$ and convergence in $L^2[0,1].$ The result follows by integrating both sides together with the fact that $\|h\|_{L^2[0,1]} \leq \|h'\|_{L^2[0,1]},$ whenever $h' \in L^2[0,1]$ and $h(0)=0$.
\end{proof}

In the univariate case, expansion with respect to the Faber-Schauder {system} is the same as piecewise linear interpolation on a dyadic grid. On the lowest resolution level, the function $f$ is approximated by $f(0) + (f(1)-f(0))x$ which interpolates the function at $x=0$ and $x=1.$ If we go up to the $j$-th resolution level, the Faber-Schauder reconstruction linearly interpolates the function $f$ on the dyadic grid $x=k/2^{j+1},$ $k=0,1, \ldots, 2^{j+1}.$ The Faber-Schauder system is consequently also a basis for the continuous functions equipped with the uniform norm. In neural networks terminology, this means that an universal approximation theorem holds.

The multivariate Faber-Schauder {system} can be obtained by tensorization. Let $\Psi_{-2}:=\phi,$ $\Psi_{-1}:=\Phi,$ $\Lambda_\ell:=\{-2,-1,(j,k), j=0,1,\ldots \ell; k=0,\ldots, 2^j-1\},$ and $\Lambda=\Lambda_{\infty}.$  Then, we can rewrite the univariate Faber-Schauder {system} on $[0,1]$ as $\{\Psi_\lambda : \lambda \in \Lambda\}.$ The $d$-variate Faber-Schauder {system} on $[0,1]^d$ consists of the functions
\begin{align*}
	\big\{ \Psi_{\lambda_1} \otimes \ldots \otimes \Psi_{\lambda_d} \, : \lambda_i\in \Lambda \big\}.
\end{align*}
In the following, we write $\blambda$ for the vector $(\lambda_1, \ldots, \lambda_d).$ Given $\blambda=(\lambda_1, \ldots, \lambda_d),$ let $I(\blambda)$ denote the set containing all indices with $\lambda_i \neq -2.$ Given an index set $I \subseteq \{1, \ldots, d\},$ $f (0_{I^c}, u_I)$ denotes the function $f$ with components in $I^c= \{1, \ldots, d\}\setminus I$ being zero. Moreover, for $I=\{i_1, \ldots, i_\ell\},$ $\partial_I= \partial_{i_1}\ldots \partial_{i_\ell}$ stands for the partial derivative with respect to the indices $i_1, \ldots, i_\ell.$ Let also $\psi_{-1}:=\phi.$

\begin{lem}
Suppose that for all sets $I\subseteq \{1, \ldots, d\},$ $\partial_I f (0_{I^c}, u_I) \in L^2[0,1]^{| I |}.$ Then, with $\bs x=(x_1, \ldots, x_d) \in [0,1]^d,$
\begin{align*}
	f(\bs x) = \sum_{\blambda \in \Lambda\times \ldots \times \Lambda} d_{\blambda} \Psi_{\lambda_1}(x_1) \cdot \ldots \cdot \Psi_{\lambda_d}(x_d)
\end{align*}
with $d_{\blambda} = \int_0^1 \ldots \int_0^1 \partial_{I(\blambda)} f(0_{I(\blambda)^c}, u_{I(\blambda)}) \prod_{i \in I(\blambda)} \psi_{\lambda_i}(u_i) du_{i_1}\hdots du_{i_\ell}.$ The right hand side should be interpreted as $\lim_{\ell \rightarrow \infty}\sum_{\blambda \in \Lambda_{\ell} \times \ldots \times \Lambda_{\ell}}$ and converges in $L^2([0,1]^d)$.
\end{lem}

\begin{proof}
By induction on $d,$
\begin{align*}
	f(\bs x) = \sum_{I =\{i_1, \ldots, i_{\ell} \}\subseteq \{1, \ldots, d\}} 
	\int_0^{x_{i_1}}\ldots \int_0^{x_{i_{\ell}}} \partial_I f (0_{I^c}, u_I)  du_{i_1}\ldots du_{i_{\ell}}.
\end{align*}
Now, we can expand each of the partial derivatives with respect to the Haar wavelet
\begin{align*}
	\partial_I f (0_{I^c}, u_I) = \sum_{\blambda = (\lambda_1, \ldots, \lambda_{\ell}) \in (\Lambda \setminus \{-2\})^\ell} d_{\blambda} \psi_{\lambda_1}(u_{i_1}) \cdot\ldots\cdot \psi_{\lambda_{\ell}}(u_{i_\ell})
\end{align*}
using that $\psi_{-1}:=\phi$ and $(\Lambda \setminus \{-2\})^\ell := (\Lambda \setminus \{-2\}) \times \ldots \times (\Lambda \setminus \{-2\}).$ Convergence is in $L^2[0,1]^{\ell}$ for the limit $\lim_{k \rightarrow \infty}\sum_{\blambda \in (\Lambda_k \setminus \{-2\})^\ell}.$ Therefore, for $\bs x=(x_1, \ldots, x_d) \in [0,1]^d,$
\begin{align*}
	f(\bs x) 
	&= \sum_{I =\{i_1, \ldots, i_{\ell} \}\subseteq \{1, \ldots, d\}} \ \ 
	\sum_{\blambda = (\lambda_1, \ldots, \lambda_{\ell}) \in (\Lambda \setminus \{-2\})^\ell} d_{\blambda} \Psi_{\lambda_1}(x_{i_1}) \cdot\ldots\cdot \Psi_{\lambda_{\ell}}(x_{i_\ell}) \\
	&= \sum_{\blambda \in \Lambda\times \ldots \times \Lambda} d_{\blambda} \Psi_{\lambda_1}(x_1) \cdot \ldots \cdot \Psi_{\lambda_d}(x_d)
\end{align*}
using that $\Psi_{-2}=1$ on $[0,1].$
\end{proof}

Because the Haar wavelet functions are piecewise constant, the coefficients $d_{\blambda}$ can also be expressed as a linear combination of function values of $f.$ In particular, we recover for $d=1$ that $d_{\lambda}=-2^{j/2}(f(\tfrac{k+1}{2^j})- 2f(\tfrac{k+1/2}{2^j})+ f(\tfrac{k}{2^j})) $ if $\lambda =(j,k).$

Similarly as in the univariate case, the multivariate Faber-Schauder system interpolates the function on a dyadic grid. Given that $f(\bs x) = \sum_{\blambda \in \Lambda\times \ldots \times \Lambda} d_{\blambda} \Psi_{\lambda_1}(x_1) \cdot \ldots \cdot \Psi_{\lambda_d}(x_d),$ let
$$T_\ell f(\bs x) = \sum_{\blambda \in \Lambda_\ell \times \ldots \times \Lambda_\ell} d_{\blambda} \Psi_{\lambda_1}(x_1) \cdot \ldots \cdot \Psi_{\lambda_d}(x_d).$$ It can be shown that $T_\ell f(\bs x) = f(\bs x)$ for all grid points $\bs x =(x_1, \ldots, x_d)$ with each component $x_j \in \{k2^{-\ell-1}\, : \, k=0, \ldots, 2^{\ell+1}\}.$ Since $T_\ell f(\bs x)$ is also piecewise linear in each component, $T_\ell f$ coincides with the linear interpolating spline for dyadic knot points. Another consequence is that the universal approximation theorem also holds for the multivariate Faber-Schauder class on $[0,1]^d.$

\begin{defi}[Faber-Schauder class]
For any non-negative integer $M$ and any constant $C>0$ the $d$-variate Faber-Schauder class truncated at resolution level $M$ is defined as
\begin{align*}
\Fs(M,C)=\Big\{\bs x\mapsto \sum_{\blambda \in \Lambda_M\times \ldots \times \Lambda_M} \beta_{\blambda} \Psi_{\lambda_1}(x_1) \cdot \ldots \cdot \Psi_{\lambda_d}(x_d) \, : \,   |\beta_{\blambda}| \leq C \Big\}.
\end{align*}
We further assume that $\|f\|_\infty\leq F$ for all $f\in\Fs(M,C)$ for some $F>0$.
\end{defi}

The cardinality of the index set $\Lambda_M$ is $(1+2^{M+1}).$ The {class} $\Fs(M,C)$ therefore consists of functions with 
\begin{align}
	I:= (1+2^M)^d	
	\label{eq.I_def}
\end{align}
many parameters. Every function $\Psi_{\lambda_1} \cdot \ldots \cdot \Psi_{\lambda_d}$ can be represented as a linear combination of $3^d$ MARS functions. This yields the following result.

\begin{lem}
\label{lem.FS_embedding}
Let $I$ be defined as in \eqref{eq.I_def}. Then, for any $M, C>0,$
\begin{align*}
	\Fs(M, C) \subseteq \Fm(3^d I, 2^{d(M+2)/2}C).
\end{align*}
\end{lem}

\begin{proof}
For $j=0,1,\ldots,$ $k=0, \ldots, 2^j-1,$ 
\begin{align*}
	\Psi_{jk}(x)
	= 
	2^{j/2}\Big( x -\frac{k}{2^j}\Big)_+
	-2^{j/2+1}\Big( x -\frac{k+1/2}{2^j}\Big)_+ + 2^{j/2}\Big( x -\frac{k+1}{2^j}\Big)_+
\end{align*}
This shows that any function $\beta_{\blambda} \Psi_{\lambda_1}(x_1) \cdot \ldots \cdot \Psi_{\lambda_d}(x_d)$ with $\blambda \in \Lambda_M\times \ldots \times \Lambda_M$ can be written as a linear combination of at most $3^d$ MARS functions with coefficients bounded by $|\beta_{\blambda}| 2^{d(M+2)/2}.$
\end{proof}

A consequence is that also the MARS functions satisfy an universal approximation theorem.

{\bf Reconstruction with respect to the Faber-Schauder class:} Given a sample $(\bX_i,Y_i),$ $i=1, \ldots,n,$ we can fit the best least-squares Faber-Schauder approximation to the data by solving
\begin{align*}
	\widehat f \in \argmin_{f\in \Fs(M, \infty)} \sum_{i=1}^n \big(\mathrm Y_i -f(\bX_i) \big)^2.
\end{align*}
The coefficients can be equivalently obtained by the least-squares reconstruction in the linear model $\mathbf{Y}= \mathbb{X} \bbeta +\beps$ with observation vector $\mathbf{Y} =(\mathrm Y_1, \ldots, \mathrm Y_n)^\top,$ coefficient vector $\bbeta=(\beta_{\blambda})_{\blambda \in \Lambda_M \times \ldots \times \Lambda_M}^\top$ and $n \times  I$ design matrix 
\begin{align*}
	\mathbb{X} = \Big( \Psi_{\lambda_1}(\mathrm X_{i1}) \cdot \ldots \cdot \Psi_{\lambda_d}(\mathrm X_{id}) \Big)_{i=1,\ldots,n; (\lambda_1, \ldots, \lambda_d) \in \Lambda_M\times \ldots \times \Lambda_M}.
\end{align*}
The matrix $\mathbb{X}$ is sparse because of disjoint supports of the Faber-Schauder functions. The least squares estimate is $\widehat{\bbeta} = (\mathbb{X}^\top\mathbb{X})^{-1}\mathbb{X}^\top \mathbf{Y}.$ If $\mathbb{X}^\top\mathbb{X}$ is singular or close to singular, we can employ Tikhonov regularization/ ridge regression leading to $\widehat{\bbeta}_\alpha = (\mathbb{X}^\top\mathbb{X}+ \alpha\mathbb I)^{-1}\mathbb{X}^\top \mathbf{Y}$ with $\mathbb I$ the identity matrix and $\alpha >0.$ For large matrix dimension, direct inversion is computationally infeasible. Instead one can employ stochastic gradient descent which has a particularly simple form and coincides for this problem with the randomized Kaczmarz method, \cite{Needell2014}.

\section{Deep neural networks (DNNs)}\label{Sec3}

Throughout the article, we consider deep ReLU networks. This means that the activation function is chosen as $\sigma(x)=(x)_+=\max(x,0),\;x\in\R.$ In computer science, feedforward DNNs are typically introduced via their representation as directed graphs. For our purposes, it is more convenient to work with the following equivalent algebraic definition. 

The number of hidden layers or depth of the multilayer neural network is denoted by $L.$ The  width of the network is described by a width vector $\bs p=(p_0,\hdots,p_{L+1})\in\N^{L+2}$ with $p_0=d$ the input dimension. For vectors $\bs y=(y_1, \ldots,y_r), \bs v=(v_1, \ldots,v_r)\in\R^r,$ define the shifted activation function $\sigma_{\bs v}:\R^r\rightarrow\R^r,$ $\sigma_{\bs v}(\bs y) = (\sigma(y_1-v_1), \ldots, \sigma(y_r-v_r))^\top.$

A network is then any function of the form 
\begin{align}
	 f(\bs x)=W_{L+1}\circ\sigma_{\bs v_L}\circ \hdots\circ W_2\circ\sigma_{\bs v_1}\circ W_1\bs x,\quad \bs x\in\R^d,
	 \label{eq.DNN_def}
\end{align}
with $W_\ell\in\R^{p_\ell\times p_{\ell-1}}$ and $\bs v_\ell\in\R^{p_\ell}.$ For convenience, we omit the composition symbol $``\circ"$ and write $f(\bs x)=W_{L+1}\sigma_{\bs v_L}\cdots W_2 \sigma_{\bs v_1}  W_1\bs x$. To obtain networks that are real-valued functions, we must have $p_{L+1}=1$.

During the network learning, regularization is induced by combining different methods such as weight decay, batch normalization and dropout. A theoretical description of these techniques seems, however, out of reach at the moment. To derive theory,  sparsely connected networks are commonly assumed which are believed to have some similarities with regularized network reconstructions. Under the sparsity paradigm it is assumed that most of the network parameters are zero or non-active. The number of {\it active/non-zero parameters} is defined as
\begin{align*}
s=\sum_{\ell=1}^{L+1}|W_\ell|_0+\sum_{\ell=1}^L|\bs v_\ell|_0,
\end{align*}
where $|W_\ell|_0$ and $|\bs v_\ell|_0$ denote the number of non zero entries  of the weight matrices $W_\ell$ and shift vectors $\bs v_\ell$. The function class of all networks with the same network architecture $(L, \bs p)$ and network sparsity $s$ can then be defined as follows.

\begin{defi}[ReLU network class]
For $L,s\in\N,$ $F>0,$ and $\bs p=(p_0,\hdots,p_{L+1})\in\N^{L+2}$ denote by $\ca F(L,\bs p,s)$ the class of $s$-sparse ReLU networks $f$ of the form \eqref{eq.DNN_def} for which the absolute value of all parameters, that is, the entries of $W_\ell$ and $\bs v_\ell$, are bounded by one and $\|f\|_\infty\leq F.$
\end{defi}

Network initialization is typically done by sampling small parameter weights. If one would initialize the weights with large values, learning is known to perform badly (\cite{Goodfellow-et-al-2016}, Section 8.4). In practice bounded parameters are therefore fed into the learning procedure as a prior assumption via network initialization. We have incorporated this in the function class by assuming that all network parameters are bounded by a universal constant which for convenience is taken to be one here.

Fitting a DNN to data is typically done via stochastic gradient descent of the empirical loss. The practical success of DNNs is partially due to a  combination of several methods and tricks that are build into the algorithm. An excellent treatment is \cite{Goodfellow-et-al-2016}.

\section{Deep networks vs. spline type methods}
\label{sec.comparison}

\subsection{Improper learning of spline type methods by deep networks}\label{sec31}
We derive below  ReLU network architectures which allow for an approximation of every MARS and Faber-Schauder function on the unit cube $[0,1]^d$ up to a small error.

\begin{lem}\label{3}
For fixed input dimension, any positive integer $M,$ any constant $C,$ and any $\ve \in (0,1],$ there exists $(L,\bs p,s)$ with
\begin{align*}
L&=O\Big(\log \frac{MC}{\ve}\Big),\quad \max_{\ell=0,\hdots,L+1} p_\ell=O(M), \quad s=O\Big(M\log \frac{MC}{\ve}\Big),
\end{align*}
such that 
\begin{align*}
\inf_{h\in\ca F(L,\bs p,s)}\|f-h\|_\infty\leq\ve \quad \text{for all} \ f\in\Fm(M,C).
\end{align*}
\end{lem}
Recall that $f \in\Fm(M,C)$ consists of $M$ basis functions. The neural network $h$ should consequently also have at least $O(M)$ active parameters. The previous result shows that for neural networks we need $O(M\log(M/\ve))$ many parameters. Using DNNs, the number of parameters is thus inflated by a factor $\log(M/\ve)$ that depends on the error level $\ve$.

Recall that $I =(1+2^M)^d.$ Using the embedding of Faber-Schauder functions into the MARS function class given in Lemma \ref{lem.FS_embedding}, we obtain as a direct consequence

\begin{lem}\label{20}
For fixed input dimension, any positive integer $M,$ any constant $C,$ and any $\ve \in (0,1],$ there exists $(L,\bs p,s)$ with
\begin{align*}
L&=O\Big(\log \frac{ \, I\, C}{\ve}\Big),\quad \max_{\ell=0,\hdots,L+1} p_\ell=O(\, I\,), \quad s=O\Big(\, I\,\log \frac{ \, I\, C}{\ve}\Big),
\end{align*}
such that 
\begin{align*}
\inf_{h\in\ca F(L,\bs p,s)}\|f-h\|_\infty\leq\ve \quad \text{for all} \ f\in\Fs(M,C).
\end{align*}
\end{lem}
The number of required network parameters $s$ needed to approximate a Faber-Schauder function with $I$ parameters is thus also of order $I$ up to a logarithmic factor.

The proofs of Lemma \ref{3} and Lemma \ref{20} are constructive. To find a good initialization of a neural network, we can therefore run MARS in a first step and then use the construction in the proof to obtain a DNN that generates the same function up to an error $\ve.$ This DNN is a good initial guess for the true input-output relation and can therefore be used as initialization for any iterative method. 

There exists a network architecture which allows to approximate the product of any two  numbers in $[0,1]$ at an exponential rate with respect to the number of network parameters, cf. \cite{Prodmult, schmidthieber2017}. This construction is the key step in the approximation theory of DNNs with ReLU activation functions and allows us to approximate the MARS and Faber-Schauder function class. More precisely, we consider the following iterative extension to the product of $r$ numbers in $[0,1]$. Notice that a fully connected DNN with architecture $(L, \mathbf{p})$ has $\sum_{\ell=0}^L(p_\ell+1)p_{\ell+1}-p_{L+1}$ network parameters.

\begin{lem}[Lemma 5 in \cite{schmidthieber2017}]\label{mmult}
For any $N \geq 1,$ there exists a network $\mathrm{Mult}_N^r\in\ca F((N+5)\lceil\log_2r\rceil,(r,6r,6r,\hdots,6r,1),s)$ such that $\mathrm{Mult}_N^r(\bs x)\in[0,1]$ and
\begin{align*}
\Big|\mathrm{Mult}_N^r(\bs x)-\prod_{j=1}^rx_j\Big|\leq 3^r2^{-N}\quad\text{for all }\bs x=(x_1,\hdots,x_r)\in[0,1]^r.
\end{align*}
The total number of network parameters $s$ is bounded by $42r^2(1+(N+5)\lceil\log_2r\rceil)$.
\end{lem}

\subsection{Lower bounds}\label{lowbo}

In this section, we show that MARS and Faber-Schauder class require much more parameters to approximate the functions $f(x) =x^2$ and $f(x_1,x_2)=(x_1+x_2-1)_+.$ This shows that the converse of Lemma \ref{3} and Lemma \ref{20} is false and also gives some insights when DNNs work better. 

{\bf The case $\bs{f(x) =x^2:}$} This is an example of a univariate function where DNNs need much fewer parameters compared to MARS and Faber-Schauder class. 

\begin{thm}
\label{thm.lb1}
Let $\eps >0$ and $f(x) =x^2.$ 
\begin{itemize}
\item[(i)] If $M \leq 1/(6\sqrt{\eps}),$ then $$\inf_{g\in \Fm(M,\infty)}\|g-f\|_\infty \geq \eps.$$ 
\item[(ii)] If $2^{M+5} \geq 1/\sqrt{\eps}$ then $$\inf_{g\in \Fs(M,\infty)}\|g-f\|_\infty \geq \eps.$$ 
\item[(iii)] There exists a neural network $h$ with $O(\log(1/\eps))$ many layers and $O(\log (1/\eps))$ many network parameters such that $$\|h-f\|_\infty \leq \eps.$$
\item[(iv)] {} (\cite{Prodmult}, Theorem 6) Given a network architecture $(L,\bs p),$ there exists a positive constant $c(L)$ such that if $s \leq c(L) \eps^{-1/(2L)},$ then $$\inf_{g\in \ca F(L, \bs p, s)}\|g-f\|_\infty \geq \eps.$$
\end{itemize}
\end{thm}

Part $(iii)$ follows directly from Lemma \ref{mmult}. For $(iv)$ one should notice that the number of hidden layers as defined in this work corresponds to $L-2$ in \cite{Prodmult}. For each of the methods, Theorem \ref{thm.lb1} allows us to find a lower bound on the required number of parameters to achieve approximation error at least $\eps.$ For MARS and Faber-Schauder class, we need at least of the order $\eps^{-1/2}$ many parameters whereas there exist DNNs with $O(\log(1/\ve))$ layers that require only $O(\log (1/\eps))$ parameters.

Parts $(i)$ and $(ii)$ rely on the following classical result for the approximation by piecewise linear functions, which for sake of completeness is stated with a proof. Similar results have been used for DNNs by \cite{Prodmult} and \cite{Liang2016}, Theorem 11.

\begin{lem}\label{pwl}
Denote by $\ca H(K)$ the space of piecewise linear functions on $[0,1]$ with at most $K$ pieces and put $f(x)=x^2.$ Then, 
\begin{align*}
	\inf_{h\in \ca H(K)} \| h- f\|_{\infty} \geq \frac{1}{8K^2}.
\end{align*}
\end{lem}

\begin{proof}[Proof of Lemma \ref{pwl}]
We prove this by contradiction assuming that there exists a function $h^*\in \ca H(K)$ with $\| h^*- f\|_\infty < 1/(8K^2).$ Because $h^*$ consists of at most $K$ pieces, there exist two values $0\leq a < b \leq 1$ with $b-a\geq 1/K$ such that $h^*$ is linear on the interval $[a,b].$ Using triangle inequality,
\begin{align*}
	\Big | h^*\Big(\frac{a+b}{2}\Big) - \Big(\frac{a+b}{2}\Big)^2 \Big| 
	&= \Big | \frac{h^*(a)+h^*(b)}{2} - \Big(\frac{a+b}{2}\Big)^2 \Big| \\
	&\geq \Big(\frac{a-b}{2}\Big)^2- \frac 12 |h^*(a) - a^2| - \frac 12 |h^*(b) - b^2| \\
	&> \frac 1{4K^2}- \frac 1{8K^2} =  \frac 1{8K^2}.
\end{align*}
This is a contradiction and the assertion follows.
\end{proof}

{\it Proof of Theorem \ref{thm.lb1} $(i)$ and $(ii):$} A univariate MARS function $f\in\Fm(M,C)$ is piecewise linear with at most $M+1$ pieces.  From Lemma \ref{pwl} it follows that $\sup_{x\in [0,1]} |f(x)-x^2|\geq1/(8(M+1)^2)\geq 1/(36 M^2)$. This yields $(i).$ To prove $(ii)$, observe that a univariate Faber-Schauder function $g\in\Fs(M,C)$ is piecewise linear with at most $2^{M+2}$ pieces. $(ii)$ thus follows again from Lemma \ref{pwl}.

{\bf The case $\bs{f(x_1,x_2) =(x_1+x_2-1)_+:}$}  This function can be realized by a neural network with one layer and four parameters. 
Multivariate MARS and Faber-Schauder functions are, however, built by tensorization of univariate functions. They are therefore not well-suited if the target function includes linear combinations over different inputs. Depending on the application this can indeed be a big disadvantage. To describe an instance where this happens, consider a dataset with input being all pixels of an image. The machine learning perspective is that a canonical regression function depends on characteristic features of the underlying object such as edges in the image. Edge filters are discrete directional derivatives and therefore linear combinations of the input. This can easily be realized by a neural network but requires far more parameters if approximated by MARS or the Faber-Schauder class.

\begin{lem}\label{aproxf}
For $f(x_1,x_2) =(x_1+x_2-1)_+$ on $[0,1]^2,$
\begin{enumerate}
\item $\displaystyle \quad\inf_{h\in\Fm(M,\infty)}\|f-h\|_\infty\geq\frac{1}{8(M+1)}$;
\item $\displaystyle\quad\inf_{h\in\Fs(M,\infty)}\|f-h\|_\infty\geq\frac{1}{8\, I}.$
\end{enumerate}

\end{lem}
\begin{proof}[Proof of Lemma \ref{aproxf}] \textit{(i)}
Let $h\in\Fm(M,\infty)$. The univariate, piecewise linear function $h(\cdot,x_2):[0,1]\rightarrow \R$ has at most $M$ kinks, where the locations of the kinks do not depend on $x_2$ (some kinks might, however, vanish for some $x_2$). Thus, there exist numbers $0\leq a<b\leq 1$ with $b-a\geq1/(M+1)$, such that $h(\cdot,x_2)$ is linear in $[a,b]$ for all $x_2\in[0,1]$. Fix $x_2=1-(a+b)/2$, such that $f(\cdot,x_2)=\sigma(\cdot-(a+b)/2).$ In particular, $f(\cdot,x_2)$ is piecewise linear in $[a,b]$ with one kink, and $f(a,x_2)=f((a+b)/2,x_2)=0,$ $f(b,x_2)=(b-a)/2.$ Therefore,
\begin{align*}
\sup_{x_1\in[a,b]}\big|f(x_1,x_2)-h(x_1,x_2)\big|\geq\frac{b-a}{8} \geq \frac{1}{8(M+1)}.
\end{align*}
\textit{(ii)}
Similar to \textit{(i)} with $a=0$ and $b=2^{-M-1}$.
\end{proof}

\section{Risk comparison for nonparametric regression}\label{risk}
In nonparametric regression with random design the aim is to estimate the dependence of a response variable $\mathrm Y\in\R$ on a random vector $\bX\in[0,1]^d.$ The dataset consists of $n$ independent copies $(\mathrm Y_i,\bX_i)_{i=1,\ldots,n}$ generated from the model
\begin{align}\label{regmod}
 \mathrm Y_i=f_0(\bX_i)+\ve_i, \quad i=1, \ldots, n.
\end{align}
We assume that $f_0:[0,1]^d\rightarrow[0,F]$ for some $F>0$ and that the noise variables $\ve_i$ are independent, standard normal and independent of $(\bX_i)_i$. An estimator or reconstruction method will typically be denoted by $\widehat f.$ The performance of an estimator is measured by the prediction risk
\begin{align*}
	R(\widehat{f},f_0):=\E\big[\big(\widehat{f}(\bX)-f_0(\bX)\big)^2\big],
\end{align*}
with $\bX$ having the same distribution as $\bX_1$ but being independent of the sample $(Y_i,\bX_i)_i$.

The cross-entropy in this model coincides with the least-squares loss, see \cite{Goodfellow-et-al-2016}, p.129.
A common property of all considered methods in this article is that they aim to find a function in the respective function class minimizing the least squares loss. Gradient methods for network reconstructions will typically get stuck in a local minimum or a flat saddle point. We define the following quantity that determines the expected difference between the loss induced by a reconstruction method $\widehat f$ taking values in a class of networks $\mathcal{F}$ and the empirical risk minimizer over $\mathcal{F},$
\begin{align}
	\Delta_n(\widehat f_n, f_0, \mathcal{F}):= \E\Big[\frac 1n \sum_{i=1}^n(Y_i -\widehat f_n (\bX_i))^2 - \argmin_{f\in\mathcal{F}} \frac 1n \sum_{i=1}^n(Y_i -f(\bX_i))^2 \Big].
	\label{eq.relaxed_ERM}
\end{align}
The quantity $\Delta_n(\widehat f_n, f_0, \mathcal{F})$ allows us to separate the statistical risk induced by the selected optimization method. Obviously $\Delta_n(\widehat f_n, f_0, \mathcal{F})\geq 0$ and $\Delta_n(\widehat f_n, f_0, \mathcal{F})= 0$ if $\widehat f_n$ is the empirical risk minimizer over $\mathcal{F}.$ For convenience, we write $\Delta_n:= \Delta_n\big(\widehat f_n, f_0,  \ca F(L, \bs p, s) \big).$

\begin{thm}\label{errm}
Consider the $d$-variate nonparametric regression model \eqref{regmod} with bounded regression function $f_0$. For any positive integer $M,$ and any constant $C,$ there exists a neural network class $\ca F(L, \bs p, s)$ with $L=O(\log Mn),$ $\max_\ell p_\ell =O(M)$ and $s=O(M\log Mn),$ such that for any neural network estimator $\widehat f$ with values in the network class $\ca F(L, \bs p, s)$ and any $0< \gamma \leq 1,$
\begin{align}\label{maap}
R(\widehat{f},f_0)\leq 
	&(1+\gamma)^2 \Big(  \inf_{f\in\Fm(M,C)}\E\big[\big(f(\bX)-f_0(\bX)\big)^2\big]+\Delta_n\Big)  \notag \\
	&+O\Big(\frac{M\log^2 (Mn)\log (M+1)}{n\gamma}\Big).
\end{align}
\end{thm}
Theorem \ref{errm} gives an upper bound for the risk of the neural network estimator in terms of the best approximation over all MARS functions plus a remainder. The remainder is roughly of order $\Delta_n + M/n$, where $M$ is the maximal number of basis functions in $\Fm(M,C).$ If $M$ is small the approximation term dominates. For large $M$, the remainder can be of larger order. This does, however, not imply that neural networks are outperformed by MARS. It is conceivable that the risk of the MARS estimator also include a variance term that scales with the number of parameters over the sample size and is therefore of order $M/n.$ To determine an $M$ that balances both terms is closely related to the classical bias-variance trade-off and depends on properties of the true regression function $f.$ For a result of the same flavor, see Theorem 1 in \cite{zhang2017}. A similar result also holds for the Faber-Schauder approximation.

\begin{thm}\label{errfs}
Consider the $d$-variate nonparametric regression model \eqref{regmod} with bounded regression function $f_0$. For any positive integer $M,$ and any constant $C,$ there exists a neural network class $\ca F(L, \bs p, s)$ with $L=O(\log In),$ $\max_\ell p_\ell =O(I)$ and $s=O(I\log In),$ such that for any neural network estimator $\widehat f$ with values in the network class $\ca F(L, \bs p, s)$ and any $0< \gamma \leq 1,$
\begin{align*}
R(\widehat{f},f_0)\leq& (1+\gamma)^2\Big( \inf_{f\in\Fs(M,C)}\E\big[\big(f(\bX)-f_0(\bX)\big)^2\big]+\Delta_n\Big)+O\Big(\frac{{I}\,\log^2(In )\log {(I+1)}}{n\gamma}\Big).
\end{align*}
\end{thm}

\section{Finite sample results}\label{FSR}
This section provides a simulation-based comparison of DNNs, MARS and expansion with respect to the Faber-Schauder system. We also consider higher-order MARS (HO-MARS) using the function system \eqref{MARS2}. In particular, we study frameworks in which the higher flexibility of networks yields much better reconstructions compared to the other approaches. For MARS we rely on the py-earth package in Python. DNNs are fitted using the Adam optimizer in Keras (Tensorflow backend) with learning rate $0.001$ and decay $0.00021.$ 

As it is well-known the network performance relies heavily on the initialization of the network parameters. In the settings below, several standard initialization schemes resulted in suboptimal performance of deep networks compared with the spline based methods. One reason is that the initialization can 'deactivate' many units. This means that we generate $\mathbf{w},v$ such that for every incoming signal $\mathbf{x}$ occurring for a given dataset, $(\mathbf{w}^t\mathbf{x} +v)_+=0.$ The gradients for $\mathbf{w},v$ are then zero as well and the unit remains deactivated for all times reducing the expressiveness of the learned network. For one-dimensional inputs with values in $[0,1],$ for instance, the $j$-th unit in the first layer is of the form $(w_j x+v_j)_+.$ The unit has then the kink point in $x=-v_j/w_j.$ The unit is deactivated, in the sense above, if $-v_j/w_j > 1$ and $w_j>0$ or if $-v_j/w_j <0$ and $w_j<0.$ To increase the expressiveness of the network the most desirable case is that the kink falls into the interval $[0,1].$ In view of the inequalities above this requires that $v_j$ and $w_j$ do not have the same sign and $|w_j|\geq |v_j|.$ A similar argument holds also for the deeper layers.

The Glorot/Xavier uniform initializer suggested in \cite{pmlr-v9-glorot10a} initializes the network by setting all shifts/biases to zero and sampling all weights on layer $\ell$ independently from a uniform distribution $U[-b_\ell, b_\ell]$ with $b_\ell=\sqrt{6/(m_{\ell-1}+m_\ell)}$ and $m_\ell$ the number of nodes in the $\ell$-th layer. This causes the kink to be at zero for all units. To distribute the kink location more uniformly, we also sample the entries of the $\ell$-th shift vector independently from a uniform distribution $U[-b_\ell, b_\ell].$

It is also possible to incorporate shape information into the network initialization. If the weights are generated from a uniform distribution $U[0, b_\ell],$ all slopes are non-negative. This can be viewed as incorporating prior information that the true function increases at least linearly. In view of the argument above, it is then natural to sample the biases from $U[-b_\ell, 0].$ We call this the increasing Glorot initializer.

As described above, it might still happen that the initialization scheme deactivates many units, which might lead to extremely bad reconstructions. For a given dataset, we can identify these cases because of their large empirical loss. As often done in practice, for each dataset we re-initialize the neural network several times and pick the reconstruction with the smallest empirical loss.

{\bf The case $\bs{f_0(x) =x^2:}$} In the previous section, we have seen that MARS and Faber-Schauder {class} have worse approximation properties for the function $f_0(x)=x^2$ if compared with DNNs. It is therefore of interest to study this as a test case with simulated independent observations from the model
\begin{align}
\mathrm Y_i=&\mathrm X_i^2+\ve_i,\quad \mathrm X_i\sim \ca U[0,1],\;\ve_i\sim\ca N(0,0.2),\quad i=1,\hdots,1000.\label{sim2}
\end{align}
Since $f_0(x)=x^2$ is smooth, a good approximation should already be provided for a relatively small number of parameters in any of the considered function {classes}. Thus, we set $M=10$ for MARS and $K=3,M=10$ for higher order MARS allowing for polynomials of degree up to three. For the Faber-Schauder {class} $\ca F_{FS}(M,C)$ we set $M=3$ meaning that details up to size $2^{-3}$ can be resolved. For the neural network reconstruction, we consider dense networks with three hidden layers and five units in each hidden layer. The network parameters are initialized using the increasing Glorot initializer described above. The simulation results are summarized in Table \ref{Tabelle1}, see also Figure \ref{figure1}. Concerning the prediction risk, all methods perform equally well, meaning that the effect of the additive noise  in model \eqref{sim2} dominates the approximation properties of the function {classes}. We find that higher order MARS often yields a piecewise linear approximation of the regression function and does not include any quadratic term. 

\begin{table}[ht]
\vspace{0.2cm}
\begin{center}
{\small
\begin{tabular}{cccc}
MARS& HO-MARS &FS & DNN \\
$3.99\cdot10^{-2} \ \ (2\cdot10^{-3})$ &$3.98\cdot10^{-2} \ \ (2\cdot10^{-3})$ &$3.98\cdot10^{-2} \ \ (2\cdot10^{-3})$&$4.04\cdot10^{-2} \ \ (2\cdot10^{-3})$\\
$2.89\cdot10^{-5} \ \ (2 \cdot10^{-6} )$&$<10^{-26} \ \ \  (<10^{-25})$&$1.34\cdot10^{-6} \ \ (4\cdot10^{-8})$&$1.36\cdot10^{-6} \ \ (4\cdot10^{-7})$\\[-0.5cm]
\end{tabular}
}
\end{center}
\caption{{\it Average prediction risks in models \eqref{sim2} (upper) and \eqref{sim1} (lower) based on 100 repetitions. Standard deviations in brackets. }}\label{Tabelle1}
\end{table}

\begin{figure}[ht]
\begin{center}
\includegraphics[width=0.95\textwidth]{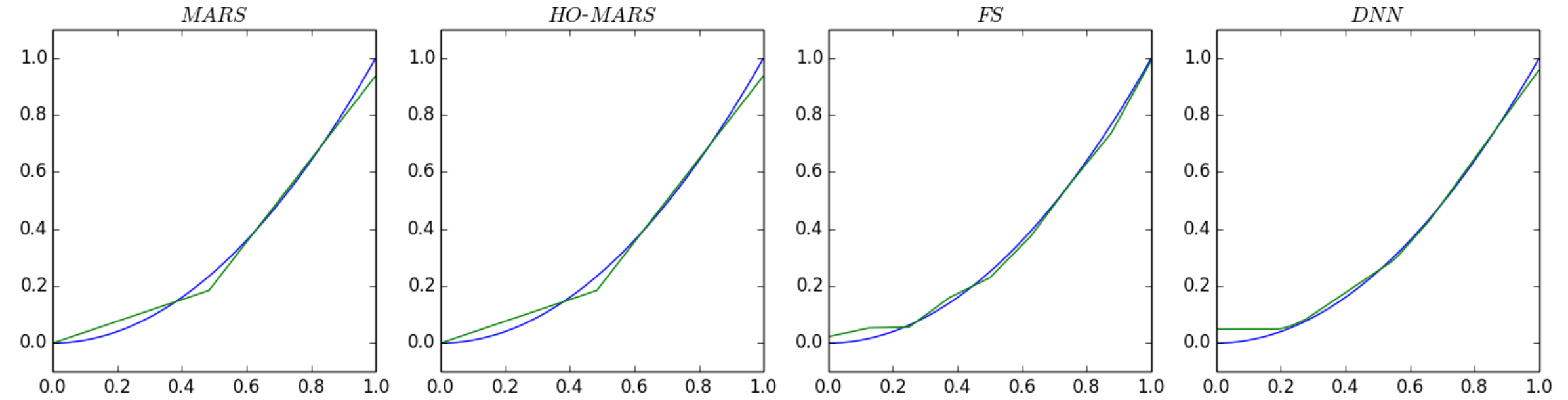}
\caption{{\it Reconstructions (green) in model \eqref{sim2}. From left to right: MARS reconstruction with $M=10;$ higher order MARS reconstruction with $K=3,M=10;$ Faber-Schauder reconstruction with $M=3$; neural network reconstruction with architecture $L=3$, $\bs p=(1,5,5,5,1).$ The regression function $f_0$ is depicted in blue.}}\label{figure1}
\end{center}
\vspace{-0.7cm}
\end{figure}

It is natural to ask about the effect of the noise on the reconstruction in model \eqref{sim2}. We therefore also study the same problem without noise, that is,
\begin{align}
\mathrm Y_i=&\mathrm X_i^2,\quad\mathrm X_i\sim \ca U[0,1] ,\quad i=1,\hdots,1000.\label{sim1}
\end{align} 
Results are summarized in Table \ref{Tabelle1}. As the underlying function class contains the square function, higher order MARS finds in this case the exact representation of the regression function. MARS, Faber-Schauder and neural networks result in a much smaller prediction risk if compared to the case with noise. DNNs and Faber-Schauder series estimation are of comparable quality while MARS performs worse.

The theory suggests that for squared regression function, DNNs should also perform better than Faber-Schauder. We believe that DNNs do not find a good local minimum in most cases explaining the discrepancy between practice and applications. For further improvements of DNNs a more refined network initialization scheme is required.

{\bf The case $\bs{f_0(x)=\sin(5\cdot 2\pi x):}$} To compare the different methods for rough functions, we simulate independent observations from
\begin{align}
\mathrm Y_i=\sin(5\cdot 2\pi \mathrm X_i),\quad \mathrm X_i\sim \ca U[0,1],\quad i=1,\hdots,1000.\label{sim3}
\end{align}
For such highly oscillating functions, all methods require many parameters. Therefore we study MARS with $M=50$, higher order MARS with $K=5$ and $M=50$, the Faber-Schauder {class} with $M=6$ and fully-connected networks with ten hidden layers, network width vector $\bs p=(1,10,\hdots,10,1)$ and Glorot uniform initializer. Results are displayed in Table \ref{Tabelle2} and Figure \ref{figure2}. Intuitively, fitting functions generated by a local function system seems to be favorable if the truth is highly oscillating and the Faber-Schauder series expansion turns out to perform best for this example.

\begin{table}[ht]
\begin{center}
{\small
\begin{tabular}{cccc}
MARS& HO-MARS & FS & DNN \\
$2.25\cdot10^{-3} \ \ (4\cdot10^{-4})$ &$9.70\cdot10^{-4}  \ \ (4\cdot10^{-4})$ &$3.94\cdot10^{-5} \ \ (2 \cdot10^{-6})$&$4.67\cdot10^{-4} \ \ (2\cdot10^{-4})$\\
$6.06\cdot10^{-4} \ \ (4\cdot10^{-4})$&$2.64\cdot10^{-4} \ \ (2\cdot10^{-4})$&$2.08\cdot10^{-3} \ \ (5\cdot10^{-4})$&$5.05\cdot10^{-4} \ \ (3\cdot10^{-4})$\\[-0.5cm]
\end{tabular}
}
\end{center}
\caption{{\it Average prediction risks in models \eqref{sim3} (upper) and \eqref{sim4} (lower) based on 100 repetitions. Standard deviations in brackets. }}\label{Tabelle2}
\end{table}
\begin{figure}[ht]
\begin{center}
\includegraphics[width=0.9\textwidth]{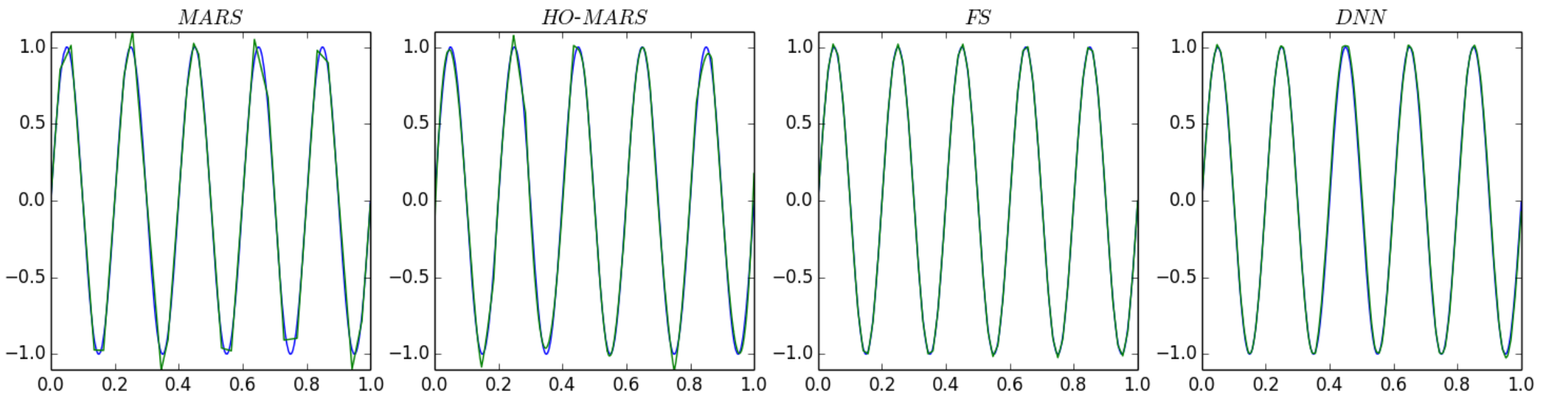}
\includegraphics[width=0.9\textwidth]{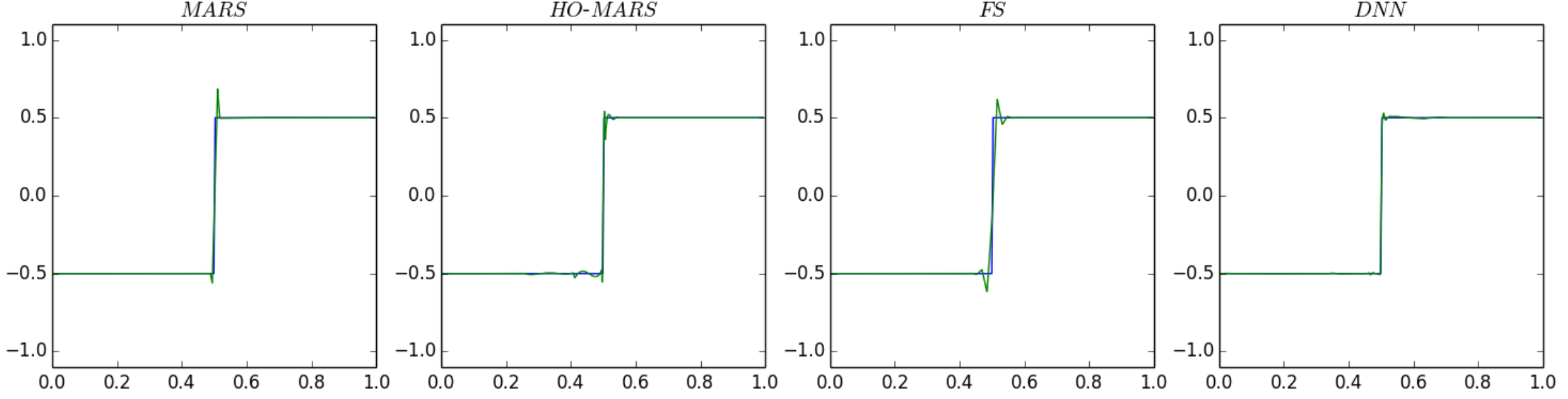}
\caption{{\it Reconstructions (green) in model \eqref{sim3} (upper) and model \eqref{sim4} (lower). From left to right: MARS reconstruction with $M=50;$ higher order MARS reconstruction with $K=5,M=50;$ Faber-Schauder reconstruction with $M=6$; neural network reconstruction with architecture $L=10$, $\bs p=(1,10,\hdots,10,1).$ $f_0$ is depicted in blue.}}
\label{figure2}
\end{center}
\vspace{-0.7cm}
\end{figure}
 
{\bf The case $\bs{f_0(x)=-1/2+\mathbbm{1}(x\geq 1/2):}$} As another toy example we consider regression with a jump function to investigate the approximation performance in a setup in which standard smoothness assumptions are violated. For this we generate independent observations from
 \begin{align}
\mathrm Y_i=\begin{cases} -\frac{1}{2},& \tn{if}\quad \mathrm X_i<\frac{1}{2},\\
\;\;\;\frac{1}{2},& \tn{if} \quad \mathrm X_i\geq\frac{1}{2},\end{cases}\quad \mathrm X_i\sim \ca U[0,1], \quad i=1,\hdots,1000.\label{sim4}
\end{align}
We use the same choices of parameters as in the simulation before. Results are shown in Table \ref{Tabelle2} and Figure \ref{figure2}, lower parts. Compared with the other methods, the neural network reconstruction does not add spurious artifacts around the jump location.

{\bf The case $\bs{f_0(x_1,x_2)=(x_1+x_2-1)_+:}$} As shown in Lemma \ref{aproxf}, DNNs need much fewer parameters than MARS or the Faber-Schauder class for this function. We generate independent observations from 
\begin{align}
\mathrm Y_i=(\mathrm X_{i,1}+\mathrm X_{i,2}-1)_+,\quad \bX_i=(\mathrm X_{i,1},\mathrm X_{i,2})^\top\sim \ca U[0,1]^2,\quad i=1,\hdots,1000.\label{sim5}
\end{align}
Notice that the regression function $f_0$ is piecewise linear with change of slope along the line $x_1+x_2=1$, $\bs x=(x_1,x_2)^\top\in[0,1]^2$.  For convenience we only consider the more flexible higher order MARS. The prediction risks reported in Table \ref{Tabelle4} and the reconstructions displayed in Figure \ref{figure4} show that in this problem neural networks outperform higher order MARS and reconstruction with respect to the Faber-Schauder {class} even if we allow for more parameters. 
\begin{table}[ht]
\begin{center}
{\small 
\begin{tabular}{ccc}
HO-MARS & FS & DNN \\
$2.20\cdot10^{-4} \ \ (7\cdot10^{-5})$&$9.86\cdot10^{-5} \ \ (9\cdot10^{-6})$&$2.46\cdot10^{-7} \ \ (1\cdot10^{-6})$\\[-0.5cm]
\end{tabular}
}
\end{center}\caption{{\it Average prediction risks in model \eqref{sim5} based on 100 repetitions. Standard deviations in brackets. }}\label{Tabelle4}
\end{table}
\begin{figure}[ht]
\begin{center}
\includegraphics[width=\textwidth]{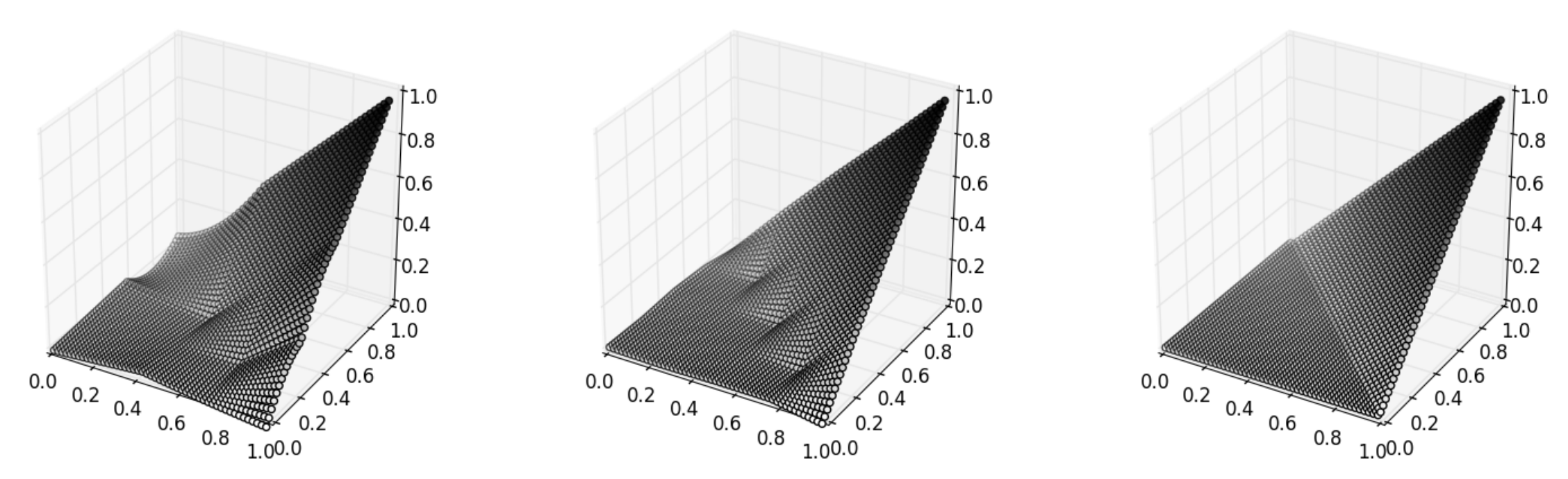}
\caption{{\it Reconstructions in model \eqref{sim5}. From left to right: Higher order MARS with $K=10,\;M=20$; Faber-Schauder reconstruction for $M=2$; neural network reconstruction for $L=1$, $\bs p=(2,1,1)$.}}\label{figure4}
\end{center}
\vspace{-0.7cm}
\end{figure}

{\bf Dependence on input dimension:} In this simulation, dependence of the risk on the input dimension is studied. We generate independent observations from
\begin{align}
\mathrm Y_i=\Big(\sum_{j=1}^d\mathrm X_{i,j}^2-\frac{d}{3}\Big)_+,\quad \bX_i=(\mathrm X_{i,1},\hdots,\mathrm X_{i,d})^\top\sim \ca U[0,1]^d,\quad i=1,\hdots,1000.\label{sim6}
\end{align}
This function is constructed in such a way that it can be well-approximated by both higher order MARS and neural networks. Independently of the dimension $d$ the setup for higher order MARS is $K=10$ and $M=30.$ We choose a densely connected neural network with architecture $L=15$, $\bs p=(d,10,\hdots,10,1)$ and Glorot uniform initializer. The prediction risks for $d=1,2,5,10$ are summarized in Table \ref{Tabelle5}. The results show that both methods suffer from the curse of dimensionality. Neural networks achieve a consistently smaller risk for all dimensions. 

\begin{table}[ht]
\begin{center}
{\small
\begin{tabular}{c|cc}
&HO-MARS & DNN \\ \hline
$d=1$&$1.47\cdot10^{-5} \ \  (3\cdot10^{-6})$&$1.33\cdot10^{-6} \ \ (4\cdot10^{-6})$\\
$d=2$
&$1.87\cdot10^{-4} \ \ (4\cdot10^{-5})$&$3.21\cdot10^{-5} \ \ (3\cdot10^{-5})$\\
$d=5$
&$8.40\cdot10^{-3} \ \ (9\cdot10^{-4})$&$1.89\cdot10^{-3} \ \ (9\cdot10^{-4})$\\
$d=10$
&$5.83\cdot10^{-2} \ \ (3\cdot10^{-3})$&$1.52\cdot10^{-2}\ \  (2\cdot10^{-3})$\\[-0.5cm]
\end{tabular}
}
\end{center}\caption{{\it Average prediction risks in model \eqref{sim6} based on 100 repetitions. Standard deviations in brackets. }}\label{Tabelle5}
\end{table}

To study the approximation performance for large input dimensions we sample independent observations  from the model
\begin{align}
\mathrm Y_i=\Big(\sum_{j=1}^{10}\mathrm X_{i,j}^2-\frac{10}{3}\Big)_+,\!\!\quad \bX_i=(\mathrm X_{i,1},\hdots,\mathrm X_{i,d})^\top\sim \ca U[0,1]^d,\!\quad i=1,\hdots,1000.\label{sim7}
\end{align}
Here, the number of active parameters 10 is fixed and small in comparison to the input dimension. For $d=1000$ the input dimension is the same as the sample size. We consider higher order MARS and neural networks with the same parameters as before. The risks are summarized in Table \ref{Tabelle6}. The obtained values are comparable with the results in the last row of Table \ref{Tabelle5} showing that both neural network and higher order MARS reconstructions are adaptive with respect to the active parameters of the model. 

\begin{table}[ht]
\begin{center}
{\small
\begin{tabular}{c|cc}
& HO-MARS & DNN \\ \hline
$d=100$&$5.81\cdot10^{-2} \ \ (3\cdot10^{-3})$&$1.32\cdot10^{-2} \ \ (3\cdot10^{-3})$\\
$d=1000$
&$5.68\cdot10^{-2} \ \ (4\cdot10^{-3})$&$3.73\cdot10^{-3} \ \ (2\cdot10^{-3})$\\[-0.5cm]
\end{tabular}
}
\end{center}\caption{{\it Average prediction risks in model \eqref{sim7} based on 100 repetitions. Standard deviations in brackets. }}\label{Tabelle6}
\end{table}

{\bf Acknowledgements.}
 We are very grateful to two referees and an associate editor for their
 constructive comments, which led to substantial improvement of  an earlier version of this manuscript.

\section{Proofs}

{\bf Proof of Lemma \ref{3}:} We first consider the approximation of one basis function $h_m$ for fixed $m\in\{1,\hdots,M\}$ by a network. By \eqref{eq.MARS_bf}, each function $h_m$ has the representation $$h_m(\bs x) = \prod_{j \in I} (s_j(x_j-t_j))_+$$ for some index set $I \subseteq \{1, \ldots, d\},$ sign vectors $s_j \in \{-1,1\}$ and shifts $\bs t = (t_j)_{j \in I} \in [0,1]^{|I|}.$ 

Now we use Lemma \ref{mmult} to construct a network $$H_{m, N}\in \ca F((N+5) \lceil \log_2 d \rceil +1, (d,d, 6d,6d, \ldots, 6d,1),s)$$ such that $\|H_{m,N}-h_m\|_\infty\leq3^d2^{-N}.$ The first hidden layer of the network $H_{m,N}$ builds the factors $(s_j(x_j-t_j))_+$ for the input $\bs x\in\R^d$. The remaining $d-|I|$ terms of the first layer are set to one. This can be achieved by choosing the first weight matrix to be
 $W_1=(s_j \mathbbm{1}(i=j, i \in I))_{i, j \in \{1,\ldots,d\}}$ and the corresponding shift vector as $\bv_1=(s_jt_j\mathbbm{1}(j\in I)- \mathbbm{1}(j \notin I))_{j },$ where $\mathbbm{1}(\cdot)$ denotes the indicator function.
 
 Because of $(s_j(x_j-t_j))_+\in [0,1]$ for $\bs x \in [0,1]^d,$ the multiplication network $\Mult_N^{d}$ introduced in Lemma \ref{mmult} can be applied to the first layer of the network. The resulting network $H_{m,N}$ is then in the class $$ \ca F((N+5) \lceil \log_2 d \rceil +1, (d,d, 6d,\ldots, 6d,1),s)$$ and satisfies $\|H_{m,N}-h_m\|_\infty\leq3^d2^{-N}.$ The number of active parameters of the network $H_{m,N}$ is $s\leq42d^2((N+5)\lceil\log_2 d\rceil+2)$.

We can simultaneously generate the networks $H_{m,N},$ $m=1,\hdots,M$, in the network space $\ca F((N+5) \lceil \log_2 d \rceil +1, (d, Md,6Md, \ldots, 6Md,M),Ms).$ Finally, we describe the construction of a network $h$ that approximates the function $f = \beta_0+ \sum_{m=1}^M \beta_m h_m$ with all coefficients bounded in absolute value by $C.$ For that we first include a layer that maps $(h_1,\hdots,h_M)^\top$ to  $(1,h_1,\hdots,h_M)^\top$ using $M+1$ active parameters. Every component is then multiplied with $C$ by parallelized sub-networks  with width 2, $2\lceil\log_2C\rceil-1$ hidden layers and $4\lceil\log_2C\rceil$ parameters each. It is straightforward to construct such a sub-network. 

For the last layer we set $W_{L+1}=C^{-1}(\beta_m)_m$. Setting $N\geq \log_2(C(M+1)3^d/\eps)$ the approximation error is 
$$\|h-f\|_\infty \leq C (M+1) 3^d2^{-N}\leq\ve.$$
The number of active parameters of $h$ is
$$M42d^2((N+5)\lceil\log_2 d\rceil+2)+2+M(4\lceil\log_2C\rceil+2).
$$
Moreover, $L=(N+5)\lceil\log_2d\rceil+ 2\lceil\log_2C\rceil+3$ and
$\max_{\ell=0,\hdots,L+1}p_\ell\leq 6Md$.
\qed

\medskip

{\bf Proofs of Theorem \ref{errm} and Theorem \ref{errfs}:} We derive an upper bound for the risk of the empirical risk minimizer using the following oracle inequality, which was proven in \cite{schmidthieber2017}. Let $\ca N(\delta,\ca F,\|\cdot\|_\infty)$ be the covering number, that is, the minimal number of $\|\cdot\|_\infty$-balls with radius $\delta$ that covers $\ca F$ (the centers do not need to be in $\ca F$).
\begin{lem}[Lemma 10 in \cite{schmidthieber2017}]\label{approx}
Consider the $d$-variate nonparametric regression model \eqref{regmod} with unknown regression function $f_0$. Let $\widehat f$ be any estimator taking values in $\ca F$ and $\{f_0\}\cup\ca F\subset\{f:[0,1]^d\rightarrow[0,F]\}$ for some $F\geq 1$. If $\Delta_n:=\Delta_n(\widehat f, f_0,\ca F)$ is as defined in \eqref{eq.relaxed_ERM} and $\ca N_n:=\ca N(\delta,\ca F,\|\cdot\|_\infty)\geq 3,$ then, for any $\delta,\gamma\in(0,1],$
\begin{align*}
R(\widehat{f},f_0)\leq(1+\gamma)^2\Big[\inf_{f\in\ca F}\E\big[\big(f(\bX)-f_0(\bX)\big)^2\big]+ \Delta_n+F^2\frac{14\log\ca N_n +20}{n\gamma}+25\delta F\Big].
\end{align*}
\end{lem}
Using the Lipschitz-continuity of the ReLU function, the covering number of a network class can be bounded as follows.
\begin{lem}[Lemma 12 in \cite{schmidthieber2017}]\label{cov}
If $V:=\prod_{\ell=0}^{L+1}(p_\ell+1)$, then, for any $\delta>0$,
\begin{align*}
\log\ca N\Big(\delta,\ca F(L,\bs p,s),\|\cdot\|_\infty\Big)\leq(s+1)\log\Big(2\delta^{-1}(L+1)V^2\Big).
\end{align*}
\end{lem} 

\begin{proof}[Proof of Theorem \ref{errm}]
Given $(M,C)$ with $C$ fixed, there exists by Lemma \ref{3} a network class $\ca F (L, \bs p, s)$ such that $L=O(\log Mn),$ $\max_{\ell} p_\ell = O(M),$ $s=O(M\log Mn)$ and such that $\inf_{h\in \ca F(L, \bs p, s)} \| h-f\|_\infty \leq 1/n.$ We will show that for this class, \eqref{maap} holds. 

We bound the metric entropy using Lemma \ref{cov}. For two positive sequences $\lesssim$ means that the right hand side is bounded by the left hand side up to a multiplicative constant. Observe that 
\begin{align*}
\log V=\log\Big(\prod_{\ell=0}^{L+1}(p_\ell+1)\Big)\lesssim L\log\Big(\max_{\ell=0,\hdots,L+1}p_\ell\Big)\lesssim L\log (M+1)
\end{align*}
yields the upper bound
\begin{align*}
	\log\ca N\Big(\frac 1n,\ca F(L,\bs p,s),\|\cdot\|_\infty\Big)
	&\lesssim M \log (Mn) \Big( \log(4nL)+ L\log (M+1) \Big) \\
	&\lesssim M\log^2 (Mn)\log (M+1).
\end{align*}
Lemma \ref{approx} and an application of Lemma \ref{3} finally yield the assertion
\begin{align*}
&R(\widehat{f},f_0)\\
&\leq (1+\gamma)^2\Big(\inf_{f\in \ca F(L,\bs p,s)}\E\big[\big(f(\bX)-f_0(\bX)\big)^2\big]+\Delta_n\Big)+O\Big(\frac{M\log^2 (Mn)\log (M+1)}{n\gamma}\Big)\\
&\leq(1+\gamma)^2\Big(\inf_{f\in\Fm(M,C)}\E\big[\big(f(\bX)-f_0(\bX)\big)^2\big]+\Delta_n\Big)+O\Big(\frac{M\log^2 (Mn)\log (M+1)}{n\gamma}\Big).
\end{align*}
\end{proof}

\begin{proof}[Proof of Theorem \ref{errfs}]
By replacing $M$ and ${I}$ in the bounds, the proof can be derived in the same way as the proof of Theorem \ref{errm}.
\end{proof}

\bibliographystyle{acm}
\bibliography{lit}
\end{document}